\documentclass{article}

\usepackage{PRIMEarxiv}

\usepackage[utf8]{inputenc} 
\usepackage[T1]{fontenc}    
\usepackage{hyperref}       
\usepackage{url}            
\usepackage{booktabs}       
\usepackage{bm}
\usepackage{amsfonts}       
\usepackage{amsthm}
\usepackage{amsmath}
\usepackage{amssymb}
\usepackage{mathtools}
\usepackage{subfigure}
\usepackage{nicefrac}       
\usepackage{microtype}      
\usepackage{lipsum}
\usepackage{fancyhdr}       
\usepackage{graphicx}       
\graphicspath{{media/}}     

\newtheorem{definition}{Definition}
\newtheorem{proposition}{Proposition}
\newtheorem{remark}{Remark}
\newtheorem{example}{Example}

\title{Fashion-Specific Attributes Interpretation \\ via Dual Gaussian Visual-Semantic Embedding
}

\author{
  Ryotaro Shimizu* \\
  Waseda University, ZOZO Research \\
  \texttt{shi3mizu8-r@fuji.waseda.jp} \\
\And
  Masanari Kimura \\
  ZOZO Research \\
  \texttt{masanari.kimura@zozo.com} \\
\And
  Masayuki goto \\
  Waseda University \\
  \texttt{masagoto@waseda.jp} \\
}

\begin{document}
\maketitle

\begin{abstract}
Several techniques to map various types of components, such as words, attributes, and images, into the embedded space have been studied. Most of them estimate the embedded representation of target entity as a point in the projective space. Some models, such as Word2Gauss, assume a probability distribution behind the embedded representation, which enables the spread or variance of the meaning of embedded target components to be captured and considered in more detail. We examine the method of estimating embedded representations as probability distributions for the interpretation of fashion-specific abstract and difficult-to-understand terms. Terms, such as ``casual,'' ``adult-casual,'' ``beauty-casual,'' and ``formal,'' are extremely subjective and abstract and are difficult for both experts and non-experts to understand, which discourages users from trying new fashion. We propose an end-to-end model called dual Gaussian visual-semantic embedding, which maps images and attributes in the same projective space and enables the interpretation of the meaning of these terms by its broad applications. We demonstrate the effectiveness of the proposed method through multifaceted experiments involving image and attribute mapping, image retrieval and re-ordering techniques, and a detailed theoretical/empirical discussion of the distance measure included in the loss function.
\end{abstract}

\keywords{Fashion Intelligence \and Fashion Interpretation \and Visual-Semantic Embedding \and Outfit Image \and Gaussian Embedding}

\section{Introduction}
It is now common for users to refer to other people's fashion outfit images through e-commerce sites and SNS and to incorporate them into their own fashion and purchasing activities. Therefore, it is important for the fashion industry to support users’ online search for fashion images to increase their interest in fashion and willingness to purchase. In the fashion field, coordination and items are described in subjective and abstract expressions such as ``casual,'' ``adult casual,'' ``beautiful casual,'' and ``formal.'' Many users perceive these expressions with difficulty, which discourages them from trying new fashion.

\quad In response to this problem, Shimizu et al.~\cite{Shimizu2022_FashionIntelligenceSystem} proposed a system to support the interpretation of these terms through various application systems that apply visual-semantic embedding (VSE). In this system, fashion images and attributes (tags) are mapped in the same embedding space to obtain an embedded representation. However, images and tags are processed in such a way that they are mapped to only one point in the embedding space. In other words, abstract tags, such as casual, which are considered to have a wide range of meanings that can be interpreted differently by each person, and concrete tags, such as jeans, are treated in the same way and mapped to only one point. This is desirable because it does not represent the breadth of the meaning (diversity) that each tag or image has in the real world in the destination space.

\quad In contrast, Gaussian embedding, as represented by Word2Gauss~\cite{Luke2015_word2gauss}, embed each element as a probability distribution, assuming a (multidimensional) Gaussian distribution behind each embedded representation of the mapped object. This method is expected to contribute to the quantification and interpretation of terms that are subjective, abstract, and not easy to interpret, which are unique to fashion. In this study, we propose dual Gaussian VSE (DGVSE), which embeds images and tags in the same space as a multidimensional Gaussian distribution. This method represents the embedding representation of each image and tag as a multidimensional Gaussian distribution, and the parameters of the embedded representation (mean vector and covariance matrix) estimated from the end-to-end model are used to enable multifaceted applications. We experimentally demonstrate the effectiveness of the proposed model using data accumulated in real services and presenting various functions that support the reduction of fashion-specific interpretation difficulties using the results obtained. The effectiveness of the proposed model is also demonstrated through a theoretical and an empirical consideration of the several types of distances that can be included in the loss function.

\quad The main contributions of this study are summarized as follows: 1) We propose a DGVSE model that can embed images and tags as probability distributions in the same space with a new end-to-end architecture. 2) Through analysis experiments using real data, we present various applications of DGVSE to support user fashion interpretation. 3) Theoretical and empirical considerations of several types of distances included in the loss function show the effectiveness of the proposed model.

\section{Related Research}
\subsection{Visual-Semantic Embedding}
VSE models for image retrieval~\cite{VSE_image2text_text2image, VSEpp, vse_sketch_based_retrieval}, visual question-answering~\cite{VSE_QA}, hashing task~\cite{VSE_hashing, vse_hashing2}, zero-shot learning~\cite{vse_zeroshotlearning1, vse_zeroshotlearning2}, person re-identification~\cite{VSE_PersonIdentification},
and image descriptions generation~\cite{VSE_disc_gen} have been widely researched.

\quad In the fashion domain, VSE is used for text-based and individual clothes image retrieval~\cite{VSE_ClothRetrieval, VSE_ClothRetrieval2, vse_fashion_image_retrieval}, and learning outfit compatibility (individual outfit item matching)~\cite{VSE_ClothMatching}. VSE included in~\cite{HAN2017_FASHIONCONCEPTDISCOVERY} is a method of embedding a fashion item image and a specific word contained in the item description in the same projective space. As a result, VSE in~\cite{HAN2017_FASHIONCONCEPTDISCOVERY} makes it possible to search for fashion images by calculating similarities between individual item images and simple words and to find specific points on the target item image that are highly related to the word (thereby creating an attribute activation map). These multiple functions constitute the advantage of VSE.

\quad Shimizu et al.~\cite{Shimizu2022_FashionIntelligenceSystem} proposed a VSE model with full-body garment images and abstract and rich tag embedding. We propose a model that enables the embedding of full-body outfit images and abstract rich tags. This study includes new techniques, such as foreground-centered learning and background regularization, to accurately reflect the relationship between full-body outfit images containing various items and rich tags in the embedding space. Various applications that quantify fashion-specific abstract expressions and support user interpretation of fashion using the embedded representations obtained from VSE were proposed.

\quad However, these VSE models estimate the embedded representation corresponding to each data as a single point in the destination projective space; they do not estimate it as a probability distribution. In this study, we propose a model that estimates the parameters of a multidimensional Gaussian distribution by assuming a multidimensional Gaussian distribution behind the embedded representations of images and tags.

\subsection{Gaussian Embedding}
Gaussian embedding is a method that solves the problem of acquiring embedded representations by estimating the embedded representations as a (multidimensional) Gaussian distribution. Luke et al. proposed Word2Gauss~\cite{Luke2015_word2gauss}, a model to solve the problem that the embedded representations obtained with the existing Word2Vec model~\cite{Mikolov2013_word2vec_arxiv, Mikolov2013_word2vec_nips} failed to consider the spread of word meanings. For example, the words ``Bach,'' ``composer,'' and ``man'' should be ordered in order of breadth of meaning, ``man > composer > Bach,'' and this relationship is automatically captured. The strength of this method is that it can estimate the variance, or breadth of meaning of words, and by observing the estimated variance, it is possible to obtain knowledge that cannot be obtained with Word2Vec. Research has developed Word2Vec into a method for embedding concepts (such as ``animal'') and words (such as ``dog'' or ``cat'') as probability distributions,~\cite{conceptW2G}.

\quad Many studies that extend network embedding to Gaussian embedding have also been published~\cite{graph_ge1, graph_ge2, graph_ge3}.
Other studies related to Gaussian embedding have developed to image recognition~\cite{visrec_ge}. Face recognition is one of the major tasks in which Gaussian embedding shows great contributions~\cite{facerec_ge, faceembedding_ge1, faceembedding_ge2}. Several studies aim to use this strength of Gaussian embedding to obtain the variance of an item in marketing areas, such as recommendation systems, and to gain important knowledge~\cite{item2gauss_marketing, ge_recommendation1, ge_recommendation2}.

\quad Some studies have extended the VSE model to Gaussian embedding. Ren et al.~\cite{Ren2016_GaussianVSE} proposed Gaussian visual-semantic embedding (GVSE). In this model, the embedded representation of a word is first learned by pre-training with GloVe~\cite{Jeffrey2014_GloVe} (focusing only on the text data attached to the image). Then, the embedding representation obtained from the pre-training is fixed as the mean vector of word embeddings, and all other parameters in the GVSE (including the covariance matrix) are estimated. In the loss function, the Mahalanobis distance is used to measure the distance between words (probability distributions) and images (points). However, the problems with this model are that it is not an end-to-end model and does not consider any image information when learning embedded representations for words, and it uses the Mahalanobis distance as the distance measure in the loss function (see below). To use methods, such as GloVe, for pre-training, it is necessary to have a situation in which a large number of words are assigned to each data item, as is the case with text data. Although the data in this study are tagged, the number of tags per image is not as large as the number of words in text data, making it difficult to learn with these methods.
Mukherjee and Hospedales~\cite{vislinguisticrec_ge} proposed a method similar to GVSE at very close to the time GVSE was proposed. That method also pre-trained the mean vector and covariance matrix for the word's embedded representation by Word2Gauss with the dataset specific to text data such as Wikipedia corpus~\cite{wikipedia_corpus}. Therefore, same as GVSE, this method also encounters problems because it is not end-to-end learning and those related to the difference between tags and sentence (words) data.

\quad In contrast, to solve all these problems, we propose DGVSE, which is an end-to-end model that takes image information into account when learning embedded representations of words and uses a measure other than the Mahalanobis distance for the distance included in the loss function. The contributions of the paper are its detailed analysis of variance, which has not been done in previous studies, and showing various applications of DGVSE.

\section{Distance Functions Inducing Embedded Spaces}
\label{sec:theory_mahalanobis}
In this section, we consider the distance type to be adopted when mapping images and attributes to the embedding space in the proposed model.

\quad Let $(\mathcal{X}, \mathcal{F}, \mathcal{\nu})$ be a measure space, where $\mathcal{X} \subseteq \mathbb{R}^{d}$ denotes the sample space, $d\in\mathbb{N}$ is the dimension of the sample space, $\mathcal{F}$ is the $\sigma$-algebra of measurable events, and $\nu$ is a positive probability measure. The set of the positive probability measure $\mathcal{P}$ is defined as
\begin{align}
    \mathcal{P} = \left\{ f(\bm{x}) \middle| f(\bm{x}) \geq 0\ (\forall{\bm{x}}\in\mathcal{X})\ \text{and}\ \int_{\mathcal{X}}f(\bm{x})d\nu(\bm{x})=1 \right\}.
\end{align}
In the following, we assume that $d\nu(\bm{x}) = d\nu = d\bm{x}$.

\begin{definition}[Mahalanobis Distance]
\label{def:mahalanobis_distance}
\quad Let $P\in\mathcal{P}$ be the probability distribution with mean vertical vector $\bm{\mu} \in \mathbb{R}^{d}$ and positive-definite covariance matrix $\bm{\Sigma}\in\mathbb{R}^{d \times d}$.
The Mahalanobis distance $d_M: \mathcal{P} \times \mathcal{X} \to [0, \infty)$ between $P$ and some point (vertical vector) $\bm{x}\in \mathcal{X}$ is defined as
\begin{align}
    d_M(P, \bm{x}) \coloneqq \sqrt{(\bm{x}-\bm{\mu})^\top\bm{\Sigma}^{-1}(\bm{x}-\bm{\mu})}.
\end{align}
\end{definition}

\begin{proposition}[Closed form of Mahalanobis distance between Gaussian distributions]
\label{prop:closed_form_mahalanobis_distance}
Let $P, Q \in\mathcal{P}$ be two Gaussian distributions with mean vertical vectors $\bm{\mu}_0, \bm{\mu}_1 \in \mathbb{R}^d$ and the same positive-definite covariance matrix $\bm{\Sigma}\in\mathbb{R}^{d \times d}$.
The closed form of the Mahalanobis distance between $P$ and $Q$ is given as
\begin{align}
    d_M(P, Q) = \sqrt{(\bm{\mu}_0 - \bm{\mu}_1)^\top\bm{\Sigma}^{-1}(\bm{\mu}_0 - \bm{\mu}_1)}.
\end{align}
\end{proposition}

\quad Many methods, including the conventional GVSE method, use the Mahalanobis distance as a measure of the distance between a point and a distribution. However, the assumption that the covariance matrices of two points are identical is too strong when one wants all points in the embedding space to be Gaussian with different variances. Therefore, we consider the embedding space induced by other distance functions.

\begin{definition}[Kullback–Leibler divergence]
The Kullback–Leibler divergence or KL-divergence $D_{KL}: \mathcal{P}\times\mathcal{P}\to [0,\infty)$ is defined between two Radon–Nikodym densities $p$ and $q$ of $\nu$-absolutely continuous probability measures by
\begin{align}
    D_{KL}[p\|q] \coloneqq \int_\mathcal{X} p\ln\frac{p}{q}d\nu = \int_{\mathcal{X}}p(\bm{x})\ln\frac{p(\bm{x})}{q(\bm{x})}d\bm{x}.
\end{align}
\end{definition}

\quad Because the KL divergence is asymmetric (i.e., $D_{KL}[p\|q] \neq D_{KL}[q\|p]$), the following symmetrization is often used to treat it as a distance function.

\begin{definition}[Jeffreys divergence]
The Jeffreys divergence $D_J:\mathcal{P}\times\mathcal{P}\to [0, \infty)$ is defined between two Radon-Nikodym densities $p$ and $q$ of $\nu$-absolutely continuous probability measures by
\begin{align}
    D_J[p\|q] &\coloneqq D_{KL}[p\|q] + D_{KL}[q\|p] \nonumber \\
    &= \int_\mathcal{X} p\ln\frac{p}{q}d\nu + \int_\mathcal{X} q\ln\frac{q}{p}d\nu \nonumber \\
    &= \int_\mathcal{X} p(\bm{x})\ln\frac{p(\bm{x})}{q(\bm{x})}d\bm{x} + \int_\mathcal{X} q(\bm{x})\ln\frac{q(\bm{x})}{p(\bm{x})}d\bm{x}.
\end{align}
\end{definition}

\begin{proposition}[Closed form of Jeffreys divergence between Gaussian distributions]
\label{prop:closed_form_jeffreys_divergence}
Let $P,Q\in\mathcal{P}$ be two Gaussian distributions with mean vertical vectors $\bm{\mu}_0$, $\bm{\mu}_1 \in \mathbb{R}^d$ and positive-definite covariance matrix $\bm{\Sigma}_0, \bm{\Sigma}_1 \in \mathbb{R}^{d \times d}$.
The closed form of the Jeffreys divergence between $P$ and $Q$ is given as
\begin{align}
    D_J[P\|Q] = \frac{1}{2}(\bm{\mu}_0 - \bm{\mu}_1)^\top(\bm{\Sigma}_0 - \bm{\Sigma}_1)(\bm{\mu}_0 - \bm{\mu}_1) + \frac{1}{2}\mathrm{tr}\left(\bm{\Sigma}_1^{-1}\bm{\Sigma}_0 + \bm{\Sigma}_0^{-1}\bm{\Sigma}_1 - 2\bm{I}_d \right),
\end{align}
where $\bm{I}_d\in\mathbb{R}^{d \times d}$ is the $d$-dimensional identity matrix.
\end{proposition}
\begin{proof}
Because both $P$ and $Q$ are Gaussian distributions, the probability density functions are given by
\begin{align*}
    p(\bm{x}) &= \frac{1}{(2\pi)^{\frac{d}{2}}\left|\bm{\Sigma}_0\right|^{\frac{1}{2}}}\exp\left\{-\frac{1}{2}(\bm{x}-\bm{\mu}_0)^\top\bm{\Sigma}^{-1}_0(\bm{x}-\bm{\mu}_0)\right\}, \\
    q(\bm{x}) &= \frac{1}{(2\pi)^{\frac{d}{2}}\left|\bm{\Sigma}_1\right|^{\frac{1}{2}}}\exp\left\{-\frac{1}{2}(\bm{x}-\bm{\mu}_1)^\top\bm{\Sigma}^{-1}_1(\bm{x}-\bm{\mu}_1)\right\}.
\end{align*}
The closed form of $D_{KL}[P\|Q]$ is given as
\begin{align}
    D_{KL}[P\|Q] &= \mathbb{E}_P[\ln p(\bm{x}) - \ln q(\bm{x})] \nonumber\\
    &= \mathbb{E}_P\left[\frac{1}{2}\ln\frac{|\bm{\Sigma}_1|}{|\bm{\Sigma}_0|} - \frac{1}{2}(\bm{x}-\bm{\mu}_0)^\top\bm{\Sigma}_0^{-1}(\bm{x}-\bm{\mu}_0) + \frac{1}{2}(\bm{x}-\bm{\mu}_1)^\top\bm{\Sigma}^{-1}_1(\bm{x}-\bm{\mu}_1) \right] \nonumber\\
    &= \frac{1}{2}\ln\frac{|\bm{\Sigma}_1|}{|\bm{\Sigma}_0|} - \frac{1}{2}\mathrm{tr}\left(\bm{I}_d\right) + \frac{1}{2}\left\{(\bm{\mu}_0 - \bm{\mu}_1)^\top\bm{\Sigma}^{-1}_1(\bm{\mu}_0-\bm{\mu}_1) + \mathrm{tr}\left(\bm{\Sigma}_1^{-1}\bm{\Sigma}_0\right)\right\}
    \label{eq:closed_form_kl_p_to_q},
\end{align}
where $\mathbb{E}_P[\cdot]$ expresses the expected value for $P$. Similarly,
\begin{align}
    D_{KL}[Q\|P] = \frac{1}{2}\ln\frac{|\bm{\Sigma}_0|}{|\bm{\Sigma}_1|} - \frac{1}{2}\mathrm{tr}\left(\bm{I}_d\right) + \frac{1}{2}\left\{(\bm{\mu}_1 - \bm{\mu}_0)^\top\bm{\Sigma}^{-1}_0(\bm{\mu}_1-\bm{\mu}_0) + \mathrm{tr}\left(\bm{\Sigma}_0^{-1}\bm{\Sigma}_1\right)\right\}. \label{eq:closed_form_kl_q_to_p}
\end{align}
From Eq.~\eqref{eq:closed_form_kl_p_to_q} and \eqref{eq:closed_form_kl_q_to_p}, we have
\begin{align*}
    D_{KL}[P\|Q] + D_{KL}[Q\|P] &= - \frac{1}{2}\mathrm{tr}(\bm{I}_d) + \frac{1}{2}(\bm{\mu}_0 - \bm{\mu}_1)^\top(\bm{\Sigma}_0^{-1}+\bm{\Sigma}_1^{-1})(\bm{\mu}_0 - \bm{\mu}_1) + \frac{1}{2}\mathrm{tr}(\bm{\Sigma}_1^{-1}\bm{\Sigma}_0 + \bm{\Sigma}_0^{-1}\bm{\Sigma}_1) \\
    &= \frac{1}{2}(\bm{\mu}_0 - \bm{\mu}_1)^\top(\bm{\Sigma}_0^{-1}+\bm{\Sigma}_1^{-1})(\bm{\mu}_0 - \bm{\mu}_1) + \frac{1}{2}\mathrm{tr}(\bm{\Sigma}_1^{-1}\bm{\Sigma}_0 + \bm{\Sigma}_0^{-1}\bm{\Sigma}_1 - 2\bm{I}_d).
\end{align*}
\end{proof}

\begin{remark}
\label{remark:equality_of_mahalanobis_and_jeffreys}
From Proposition~\ref{prop:closed_form_mahalanobis_distance} and Proposition~\ref{prop:closed_form_jeffreys_divergence}, we can see that $d^2_M(P,Q) = 2 D_J[P\|Q]$ if $P$ and $Q$ are the Gaussian distributions with identical covariance matrix.
\end{remark}

\begin{remark}
It is obvious that VSE with Mahalanobis distance induces variance-agnostic embedded space.
\end{remark}

\quad Therefore, the symmetric KL divergence content in the situation where the variance-covariance matrices of the probability distributions being compared are perfectly matched coincides with the Mahalanobis distance. This means that estimating each parameter using the Mahalanobis distance has the risk of not estimating the parameters well. Specifically, when measuring the distance between distribution $P$ and distribution $Q$, the calculation is performed under the assumption that the variance of distribution $P$ coincides with that of distribution $Q$. When measuring the distance between distribution $P$ and distribution $R$, the calculation is performed under the assumption that the variance of distribution $P$ coincides with that of distribution $R$. In other words, the parameters included in the model are estimated in a situation where the variance of distribution $P$ changes in a variety of ways (variance ignorance), depending on the pairs. In this situation, there is a risk that stable parameter estimation cannot be performed.

\quad Theoretical analysis suggests that care should be taken when adopting the Mahalanobis distance as a loss function (to measure the distance between a point and a probability distribution) in machine learning methods.

\begin{definition}[Wasserstein Distance]
Given two probability distributions $P$ and $Q$ on two Polish spaces $(\mathcal{X}, d_\mathcal{X})$ and $(\mathcal{Y}, d_\mathcal{Y})$ and a positive lower semi-continuous cost function $c, \mathcal{X}\times\mathcal{Y} \to \mathbb{R}^+$, optimal transport focuses on solving the following optimization problem:
\begin{align}
    \inf_{\pi\in\Pi(P,Q)}\inf_{\mathcal{X}\times\mathcal{Y}} c(\bm{x},\bm{y})d\pi(\bm{x},\bm{y}), \label{eq:optimal_transport}
\end{align}
where $\Pi(P,Q)$ is the set of measures on $\mathcal{X}\times\mathcal{Y}$ with marginals $P$ and $Q$.
When $\mathcal{X}$ and $\mathcal{Y}$ are subspaces in $\mathbb{R}^d$ and $c(\bm{x},\bm{y})=\|\bm{x}-\bm{y}\|^l$, where $\|\bm{x}\|^l$ is $l$-norm for vector $\bm{x}$ with $l \geq 1$, Eq.~\eqref{eq:optimal_transport} induces a distance over the set of measures with finite moment of order $l$, known as the $l$-Wasserstein distance $W_l$:
\begin{align}
    W_l(P,Q) \coloneqq \left(\inf_{\pi\in\Pi(P, Q)}\int_{\mathcal{X}\times\mathcal{Y}}\|\bm{x}-\bm{y}\|^l d\pi(\bm{x},\bm{y})\right)^{\frac{1}{l}},
\end{align}
or equivalently
\begin{align}
    W^l_l(P, Q) \coloneqq \inf_{\bm{X}\sim P,\bm{Y}\sim Q}\mathbb{E}\left[\|\bm{X} - \bm{Y}\|^l\right].
\end{align}
\end{definition}

\begin{proposition}[Closed form of the Wasserstein Distance Between Gaussian Distributions~\cite{dowson1982frechet,takatsu2010wasserstein}]
Let $P,Q \in \mathcal{P}$ be two Gaussian distributions with mean vertical vectors $\bm{\mu}_0$ and $\bm{\mu}_1 \in \mathbb{R}^d$ and positive-definite covariance matrix $\bm{\Sigma}\in\mathbb{R}^{d \times d}$, where $d\in\mathbb{N}$.
The closed form of the $2$-Wasserstein distance between $P$ and $Q$ is given as
\begin{align}
    W^2_2(P,Q) \coloneqq (\bm{\mu}_0 - \bm{\mu}_1)^{\top} (\bm{\mu}_0 - \bm{\mu}_1) + \mathrm{tr}\left(\bm{\Sigma}_0 + \bm{\Sigma}_1 - 2\left(\bm{\Sigma}_0^{\frac{1}{2}}\bm{\Sigma}_1\bm{\Sigma}_0^{\frac{1}{2}}\right)\right).
\end{align}
\end{proposition}

\quad Based on the above theoretical basis as well, the proposed method assumes a multidimensional Gaussian distribution for both images and tags, allowing the use of various measures such as KL divergence content and $2$-Wassestein distance.

\section{Geometry of Embedded Space}
\label{sec:natural_params}
In general, it is known that the space of probability distributions is not an Euclidean space, but rather constitutes a Riemannian manifold~\cite{amari2000methods,amari2016information}.
This means that the space of probability distributions is a non-Euclidean space, which does not guarantee the validity of general vector operations.
Therefore, even if the embedding vectors $\bm{\mu}$ and $\bm{\Sigma}$ are obtained by DGVSE, they cannot be used directly for applications that combine multiple individual tags or embedded representations of images such as tags embedding and image retrieval.
Thus, it is necessary to apply appropriate transformations to the embedded vectors to obtain new parameter vectors to allow operations such as those allowed in Euclidean space.
Geometrically, these transformations are called coordinate transformations, and parameter coordinates that allow linear operations are called affine coordinate systems.

\quad Consider an exponential family, the generalization of the Gaussian distributions, expressed in the following form:
\begin{align}
    p(\bm{x};\bm{\theta}) \coloneqq \exp\left\{\bm{\theta}^{\top} \bm{z} + k(\bm{x}) - \psi(\bm{\theta}) \right\}, \label{eq:exponential_family}
\end{align}
where $\bm{x}$ is a random variable, $\bm{\theta}=(\theta^{(1)},\dots,\theta^{(n)})^{\top}$ is an $n$-dimensional vector parameter, $h_{i}(\bm{x})$ are $n$ functions of $\bm{x}$, which are linearly independent, $k(\bm{x})$ is a function of $\bm{x}$, and $\psi$ corresponds to the normalization factor.
Here, let $\bm{z} = (z_1,\dots,z_i,\dots,z_n)^{\top} = (h_1(\bm{x}),\dots,h_i(\bm{x}),\dots,h_n(\bm{x}))^{\top}$ be a new vector random variable and $d\nu(\bm{z})$ be a measure in the sample space $\cal{Z} \subseteq \mathbb{R}^{n}$ defined as
\begin{align}
    d\nu(\bm{z}) \coloneqq \exp\{k(\bm{x})\}d\bm{x}.
\end{align}
Then, Eq.~\eqref{eq:exponential_family} is rewritten as
\begin{align}
    p(\bm{x};\bm{\theta})d\bm{x} &= \exp\left\{\bm{\theta}^{\top} \bm{z} - \psi(\bm{\theta})\right\}d\nu(\bm{z}), \\
    p(\bm{z};\bm{\theta}) &= \exp\left\{\bm{\theta}^{\top} \bm{z} - \psi(\bm{\theta})\right\}.
\end{align}

The family of distributions $\mathcal{M}=\{p(\bm{z};\bm{\theta})\}$ forms a $J$-dimensional manifold, where $\bm{\theta}$ is a coordinate system.
Because $\psi(\bm{\theta})$ is a normalization factor, we have
\begin{align}
    \int_Z p(\bm{z};\bm{\theta})d\nu(\bm{z}) = 1,
\end{align}
and
\begin{align}
    \psi(\bm{\theta}) = \log \int_{\cal{Z}} \exp(\bm{\theta}^{\top} \bm{z})d\nu(\bm{z}).
\end{align}
Here, a dually flat Riemannian structure is introduced in $\mathcal{M}$ using $\psi(\bm{\theta})$.
The affine coordinate system is $\bm{\theta}$, which is called the natural parameter, and the dual affine parameter is given by the Legendre transformation $\bm{\theta}^* = \nabla\psi(\bm{\theta})$, which is the expectation of $\bm{z}$ denoted by $\bm{\eta}=(\eta_1,\dots,\eta_n)^{\top}$ as
\begin{align}
    \bm{\eta} \coloneqq \bm{\theta}^* = \mathbb{E}[\bm{z}] = \int_Z \bm{z} p(\bm{z};\bm{\theta})d\nu(\bm{z}).
\end{align}
Here, $\bm{\eta}$ is called the expectation parameter.
Hence, $\bm{\theta}$ and $\bm{\eta}$ are two affine coordinate systems connected by the Legendre transformation.

\begin{example}[Univariate Gaussian distribution]
The probability density function of the Gaussian distribution with mean $\mu$ and variance $\sigma^2$ is given as
\begin{align}
    p(x; \mu, \sigma) = \frac{1}{\sqrt{2\pi}\sigma}\exp\left\{-\frac{(x-\mu)^2}{2\sigma^2}\right\}. \label{eq:univariate_gaussian}
\end{align}
Let $\bm{\xi} = (\xi^{(1)}, \xi^{(2)})^{\top}$ where
\begin{align*}
    \xi^{(1)} &= h_1(x) = x, \\
    \xi^{(2)} &= h_2(x) = x^2.
\end{align*}
Here, we can see that $\xi^{(1)}$ and $\xi^{(2)}$ are dependent, but are linearly independent. We further introduce new parameters $\bm{\theta} = (\theta^{(1)}, \theta^{(2)})^{\top}$ as
\begin{align*}
    \theta^{(1)} &= \frac{\mu}{\sigma^2}, \\
    \theta^{(2)} &= -\frac{1}{2\sigma^2}.
\end{align*}
Then, Eq~\eqref{eq:univariate_gaussian} is written in the standard form as
\begin{align}
    p(\bm{\xi}; \bm{\theta}) = \exp\{\bm{\theta}^{\top} \bm{\xi} - \psi(\bm{\theta})\}.
\end{align}
The convex function $\psi(\bm{\theta})$ is given by
\begin{align*}
    \psi(\bm{\theta}) &= \frac{\mu^2}{2\sigma^2} + \log\left(\sqrt{2\pi}\sigma\right) \\
    &= -\frac{(\theta^{(1)})^2}{4\theta^{(2)}} - \frac{1}{2}\log(-\theta^{(2)}) + \frac{1}{2}\log\pi.
\end{align*}
Finally, the dual affine coordinates $\bm{\eta}$ are given as
\begin{align}
    \eta_1 = \mu, \quad \eta_2 = \mu^2 + \sigma^2.
\end{align}
\end{example}

\section{Methodology}
The DGVSE model assumes a multidimensional Gaussian distribution behind embedded representations of full-body outfit images and attributes and allows embedding each of them (as a probability distribution) in the same space while also considering the spread of meaning. The points to identify the model from the conventional GVSE are: 1) it is an end-to-end model, and 2) it considers image information when learning embedded representations of words. In addition, based on the theoretical basis in section~\ref{sec:theory_mahalanobis}, 3) it can estimate not only words but also embedded representations of images as probability distributions, and 4) the distance included in the loss function is not Mahalanobis distance. Moreover, as mentioned in section~\ref{sec:natural_params}, 5) natural parameters are introduced when combining multiple individual tag distributions.

\subsection{Model Architecture}
The model structure is shown in Figure~\ref{fig_structure}.

\begin{figure}[ht]
\centering
\includegraphics[width=0.95\linewidth]{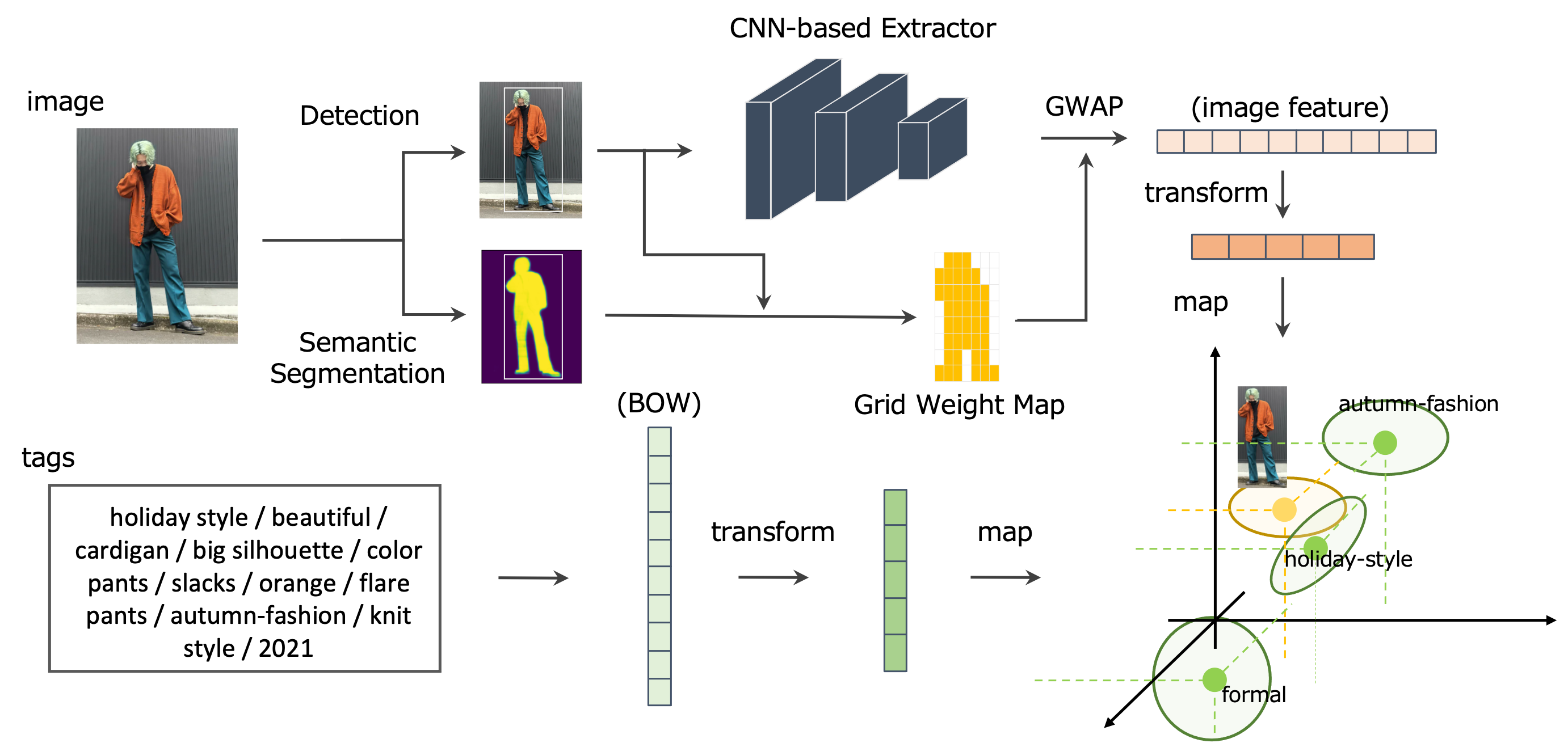}
\caption{Structure of a prototype of our dual Gaussian visual-semantic embedding model proposal \label{fig_structure}}
\end{figure}

\quad The basic structure is based on the VSE in the fashion intelligence system~\cite{Shimizu2022_FashionIntelligenceSystem}. We then assume a multidimensional Gaussian distribution behind the embedded representations of both images and tags. This allows us to estimate the mean and variance (semantic spread) for the embedded representation of each image and tag. This model solves the following problems of the conventional GVSE model: 1) it is not an end-to-end learning method, 2) it only looks at the co-occurrence of words and ignores image information when learning the embedded representation (mean) of words, and 3) it assumes a multidimensional Gaussian distribution only for words (hence the problem of measuring the distance between a point and the probability distribution using the Mahalanobis distance). In particular, the third problem is theoretically presented in the abovementioned section~\ref{sec:theory_mahalanobis}.

\subsection{Parameter Optimization}
The dataset for this study consists of a single full-body outfit image to which multiple tags are assigned.
First, embedding (foreground-centered learning) based on convolutional neural network (CNN) and grid weight map is performed on an arbitrary image $I$ to obtain an image embedded representation of foreground $\mathrm{\mathbf{x}} \sim \mathcal{N}(\bm{\mu}_{I}, \bm{\Sigma}_{I})$.
Here, $\mathrm{\mathbf{x}}, \bm{\mu}_{I} \in \mathbb{R}^{d}$, $\bm{\Sigma}_{I} \in \mathbb{R}^{d \times d}$, where $d$ is the dimension of the embedded space, $\bm{\mu}_{I}$ is mean vertical vector, and $\bm{\Sigma}_{I}$ is the assumed spherical covariance matrix for the image $I$. Furthermore, the tags set $T$ assigned to an arbitrary image $I$ is embedded to obtain the tags embedded representation $\mathrm{\mathbf{v}} \sim \mathcal{N}(\bm{\mu}_{T}, \bm{\Sigma}_{T})$.
Here, $\mathrm{\mathbf{v}}, \bm{\mu}_{T} \in \mathbb{R}^{d}$, $\bm{\Sigma}_{T} \in \mathbb{R}^{d \times d}$, where $\bm{\mu}_{T}$ is mean vertical vector and $\bm{\Sigma}_{T}$ is the assumed spherical covariance matrix for the tags $T$.

\quad In computing this tag embedding representation $\mathrm{\mathbf{v}}$, the probability distributions of embedding representations for the individual tags in arbitrary tags $T$ are combined using the abovementioned method that introduces natural parameters. We now define a tag set $T = \{t_1, \cdots, t_n, \cdots, t_{N_{T}}\}$ consisting of all $N_{T}$ individual tags. Then, the natural parameters $\Theta_n := (\bm{\theta}^{(1)}_{n}, \bm{\theta}^{(2)}_{n})^{\top} \in \mathbb{R}^{d} + \mathbb{R}^{d \times d}$ for the individual tag embedding representation $\mathrm{\mathbf{a}}_{t_n} \sim \mathcal{N}(\bm{\mu}_{t_n}, \bm{\Sigma}_{t_n})$ are calculated by the following Eq.~\eqref{eq_tags_natural_params1}-\eqref{eq_tags_natural_params2}:
\begin{eqnarray}
\bm{\theta}^{(1)}_{n} &=& \bm{\Sigma}_{t_n}^{-1} \bm{\mu}_{t_n}, \label{eq_tags_natural_params1} \\
\bm{\theta}^{(2)}_{n} &=& -\frac{1}{2} \bm{\Sigma}_{t_n}^{-1}. \label{eq_tags_natural_params2}
\end{eqnarray}
Here, $\mathrm{\mathbf{a}}_{t_n}, \bm{\mu}_{t_n} \in \mathbb{R}^{d}$, $\bm{\Sigma}_{t_n} \in \mathbb{R}^{d \times d}$, where $\bm{\mu}_{t_n}$ is a mean vertical vector and $\bm{\Sigma}_{t_n}$ is assumed a spherical covariance matrix for the individual tag $t_n$. Furthermore, natural parameters are calculated for all individual tags in $T$, and their centroid $\Theta_{T}$ is calculated. Using the calculated centroid $\Theta_{T} := (\bm{\theta}^{(1)}_{T}, \bm{\theta}^{(2)}_{T})^{\top} \in \mathbb{R}^{d} + \mathbb{R}^{d \times d}$, the following Eq.~\eqref{eq_tags_embedding_mu}-\eqref{eq_tags_embedding_cov} gives the parameters included in Gaussian distribution $\mathcal{N}(\bm{\mu}_{T}, \bm{\Sigma}_{T})$ for the embedded representation of the tag set $T$:
\begin{eqnarray}
\bm{\mu}_{T} &=& -\frac{1}{2} {\bm{\theta}^{(2)}_{T}}^{-1} \bm{\theta}^{(1)}_{T}, \label{eq_tags_embedding_mu} \\
\bm{\Sigma}_{T} &=& -\frac{1}{2} {\bm{\theta}^{(2)}_{T}}^{-1}. \label{eq_tags_embedding_cov}
\end{eqnarray}

\quad This method of computation, which introduces natural parameters, allows the distribution to be synthesized while accurately accounting for the variance of the elements (individual tags) that make up the set (tags).

\quad The basic policy of learning DGVSE is to learn so that the probability distribution of embedded representations in the image and the probability distribution of embedded representations in terms of tags attached to the image are close. Therefore, the parameter estimation is achieved by optimizing the following contrastive loss Eq.~\eqref{eq_loss}~\cite{Shimizu2022_FashionIntelligenceSystem, VSEpp,  HAN2017_FASHIONCONCEPTDISCOVERY}:

\begin{eqnarray}
\label{eq_loss}
\mathcal{L}(\mathrm{O}) &=& \sum \mathrm{max} \left(0, \mathit{m} + \mathit{d}(\mathrm{\mathbf{x}}^{+}, \mathrm{\mathbf{v}}^{+}) - \mathit{d}(\mathrm{\mathbf{x}}^{+}, \mathrm{\mathbf{v}}^{-}) \right) + \sum \mathrm{max} \left(0, \mathit{m} + \mathit{d}(\mathrm{\mathbf{v}}^{+}, \mathrm{\mathbf{x}}^{+}) - \mathit{d}(\mathrm{\mathbf{v}}^{-}, \mathrm{\mathbf{x}}^{+}) \right),
\end{eqnarray}
where $\mathrm{O} = \{ V, \mathrm{\mathbf{W}}_I, \mathrm{\mathbf{W}}_T \}$ is a set of target parameters to be optimized, $V$ is a parameter set contained in CNN, $\mathrm{\mathbf{W}}_I \in \mathbb{R}^{d \times r}$ is a transform matrix from an image feature vector obtained from a CNN-based extractor to the image embedded representation ($r$ is the number of dimensions of the final convolutional layer of the CNN), $\mathrm{\mathbf{W}}_T \in \mathbb{R}^{d \times s}$ is a transform matrix from a bag-of-words representation for tags $T$ to a tags embedded representation ($s$ is the number of tags in the entire dataset), $m$ is a margin, $d(\mathbf{x}, \mathbf{y})$ indicates the distance between vectors $\mathbf{x}$ and $\mathbf{y}$, and $\beta > 0$ is a positive hyperparameter to adjust the importance of the background regularization term. The superscript sign $+$ of $A^{+}$ indicates that $A$ is a variable related to the positive sample, and $-$ of $A^{-}$ indicates that $A$ is a variable related to the negative sample.

\quad Because we assume a multidimensional Gaussian distribution behind both the image and the embedded representation of the tag, the distance measure in the embedding space must be able to consider not only the mean but also the covariance matrix. Therefore, the following three distance measures are adopted in this study. While the Mahalanobis distance $d_{\mathrm{M}}(\mathrm{\mathbf{x}}, \mathrm{\mathbf{y}})$ is also adopted for comparison, the KL divergence $d_{\mathrm{KL}}(\mathrm{\mathbf{x}}, \mathrm{\mathbf{y}})$ and $2$-Wassestein distance $d_{\mathrm{W^2_2}}(\mathrm{\mathbf{x}}, \mathrm{\mathbf{y}})$ are adopted based on theoretical grounds to measure the distance between vectors $\mathrm{\mathbf{x}}$ and $\mathrm{\mathbf{y}}$:

\begin{eqnarray}
\label{eq_d_mahalanobis}
    d_{\mathrm{M}}(\mathrm{\mathbf{x}}, \mathrm{\mathbf{v}}) &=& \sqrt{(\bm{\mu}_{I} - \bm{\mu}_{T} )^\top \Sigma^{-1} ( \bm{\mu}_{I} - \bm{\mu}_{T} )}, \\
\label{eq_d_kldivergence}
    d_{\mathrm{KL}}(\mathrm{\mathbf{x}}, \mathrm{\mathbf{v}}) &:=& D_{\mathrm{KL}}(\mathrm{\mathbf{x}}, \mathrm{\mathbf{v}}) = \frac{1}{2}\ln\frac{|\bm{\Sigma}_{I}|}{|\bm{\Sigma}_{T}|} - \frac{1}{2}\mathrm{tr}\left(\bm{I}_d \right) + \frac{1}{2}\left\{(\bm{\mu}_{T} - \bm{\mu}_{I})^\top\bm{\Sigma}^{-1}_{I}(\bm{\mu}_{T}-\bm{\mu}_{I}) + \mathrm{tr}\left(\bm{\Sigma}_{I}^{-1}\bm{\Sigma}_{T}\right)\right\}, \\
\label{eq_d_wassestein}
    d_{\mathrm{W^2_2}}(\mathrm{\mathbf{x}}, \mathrm{\mathbf{v}}) &:=& W^2_2(\mathrm{\mathbf{x}}, \mathrm{\mathbf{v}}) = (\bm{\mu}_{I}-\bm{\mu}_{T})^{\top}(\bm{\mu}_{I}-\bm{\mu}_{T}) + \mathrm{tr}\left( \bm{\Sigma}_{I}+\bm{\Sigma}_{T}-2\left(\bm{\Sigma}_{I}^{\frac{1}{2}} \bm{\Sigma}_{T}  \bm{\Sigma}_{I}^{\frac{1}{2}} \right)^{\frac{1}{2}} \right),
\end{eqnarray}
where the joint covariance matrix $\Sigma$ in Eq.~\eqref{eq_d_mahalanobis} is defined as $\frac{1}{2}\left(\Sigma_{I} + \Sigma_{T}\right)$ for convenience in this study.

\quad The assumption of probability distributions on both sides is expected to solve the problem of dispersion ignorance that occurs when the Mahalanobis distance is used to measure the distance between a point and a distribution, which is adopted in GVSE. In subsequent sections of the experiment, we will observe and discuss how the results change with each distance measure. One of the main contributions of this study is the detailed theoretical analysis of the differences between the three typical distance scales adopted, based not only on the observed experimental results, but also on the results.

\quad DGVSE is trained through the above process of parameter optimization. By considering the probability distribution of embedded representations of the resulting images and tags, DGVSE not only retains the useful functions of the traditional fashion intelligence system, such as search and sorting, but also enables the interpretation of abstract fashion terms related to dispersion, which is not possible with the previous methods. 

\section{Experimental Analysis}
To evaluate the effectiveness of the proposed DGVSE model, we applied it to actual posted full-body outfit image data and the tags information attached to each image accumulated in WEAR~\cite{app_wear}, a fashion coordination application including SNS features.

\subsection{Experimental Settings}
The number of full-body outfit images in the experimental data was 15,740, and the number of unique tags attached to all images was 1,104. All models in the target images were female, and the backgrounds of all images contained relatively little noise. An example of a full-body outfit image and its tags is shown in Figure~\ref{fig_sample_dataset} below.

\begin{figure}[ht]
\centering
\includegraphics[width=0.95\linewidth]{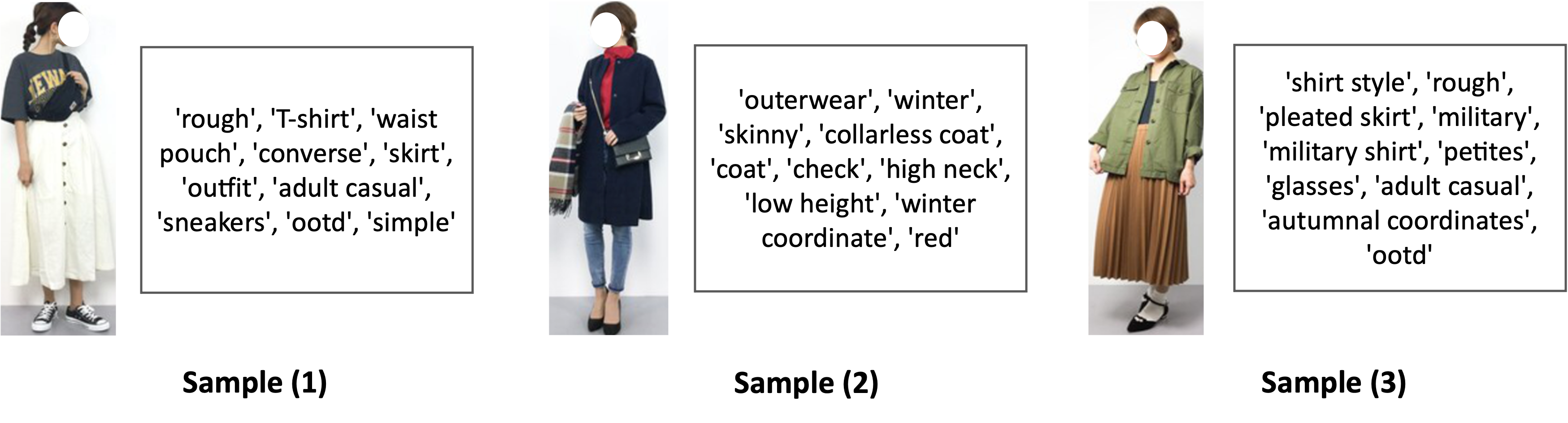}
\caption{Example of samples in the target dataset \label{fig_sample_dataset}}
\end{figure}

\quad The embedded representation dimension included in the VSE model $d$ was set to 64. The learning rate was 0.001, and the number of epochs was 50. The batch size was 32, and the margin $m$ was set to 0.2. We used GoogleNet and Inception V3 pre-trained on ImageNet~\cite{deng2009imagenet} as the extractor to assess the impact of CNN. For preprocessing, SSD (extractor: MobileNet V2~\cite{mobilenetv2}) trained on Open Images~\cite{OpenImages2}) was used for object detection, and FCN (extractor: ResNet~\cite{resnet}) trained on MS-COCO~\cite{COCO} was used for semantic segmentation.

\subsection{Attribute Mapping}
The average values of embedded tags obtained from the proposed model were compressed by t-SNE~\cite{Maaten2008_tSNE} and shown in a two-dimensional map in Figure~\ref{fig_result_mapping} below. However, because it is not possible to see the results of mapping all tags because of the scope of the figure, some representative abstract attributes are extracted and mapped.

\begin{figure}[ht]
 \begin{center}
  \subfigure[Mahalanobis distance]{
   \includegraphics[width=0.98\columnwidth]{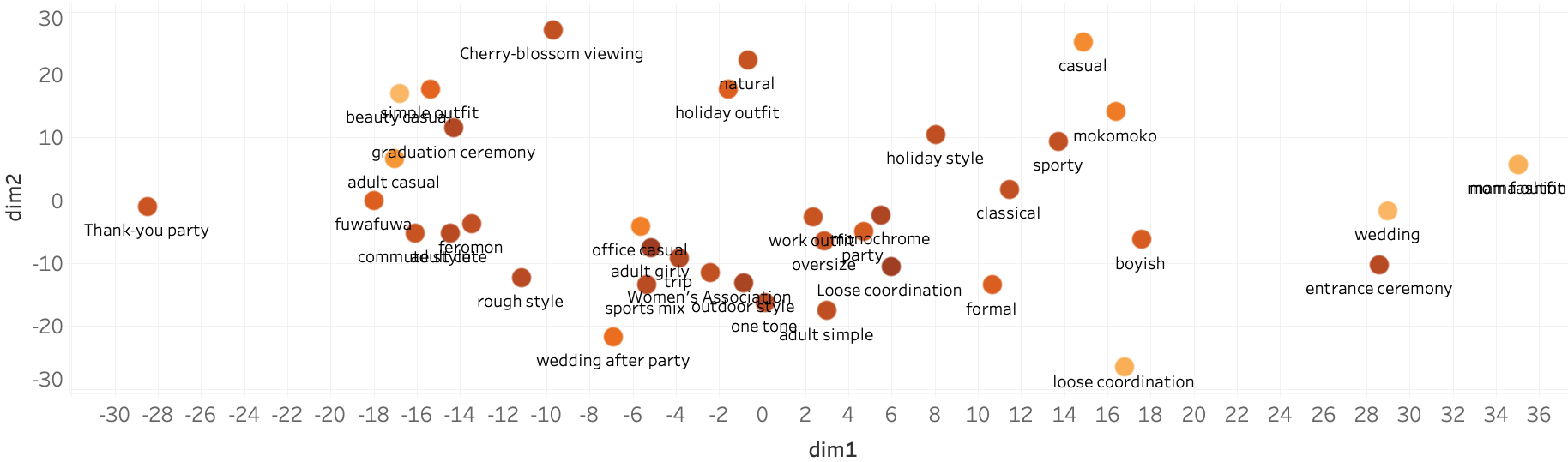}
  } \\
  \subfigure[KL-divergence]{
   \includegraphics[width=0.98\columnwidth]{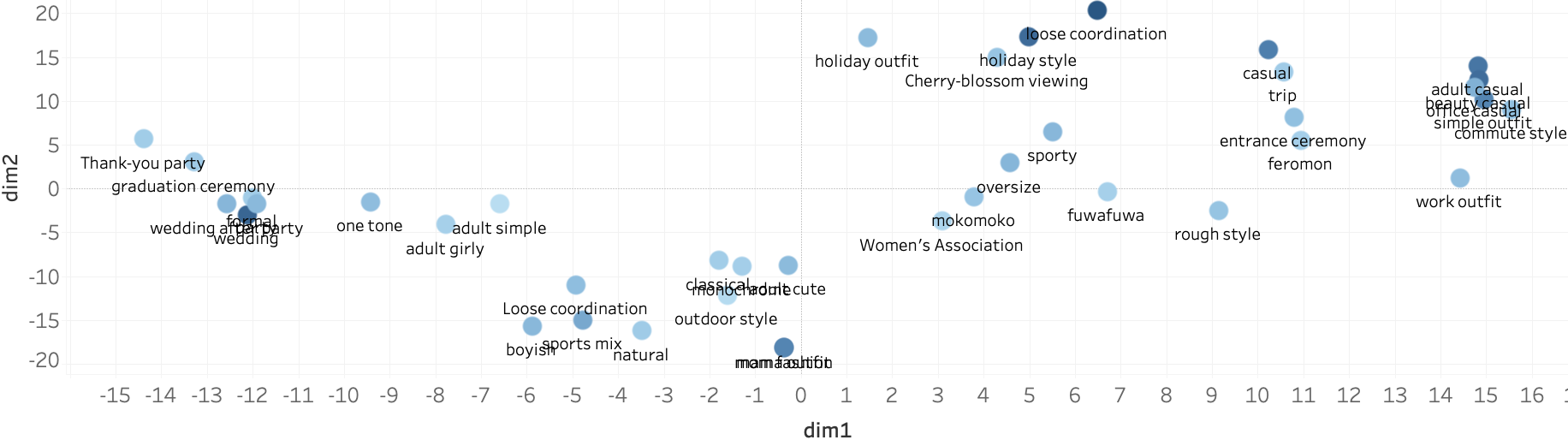}
  } \\
  \subfigure[$2$-Wassestein distance]{
   \includegraphics[width=0.98\columnwidth]{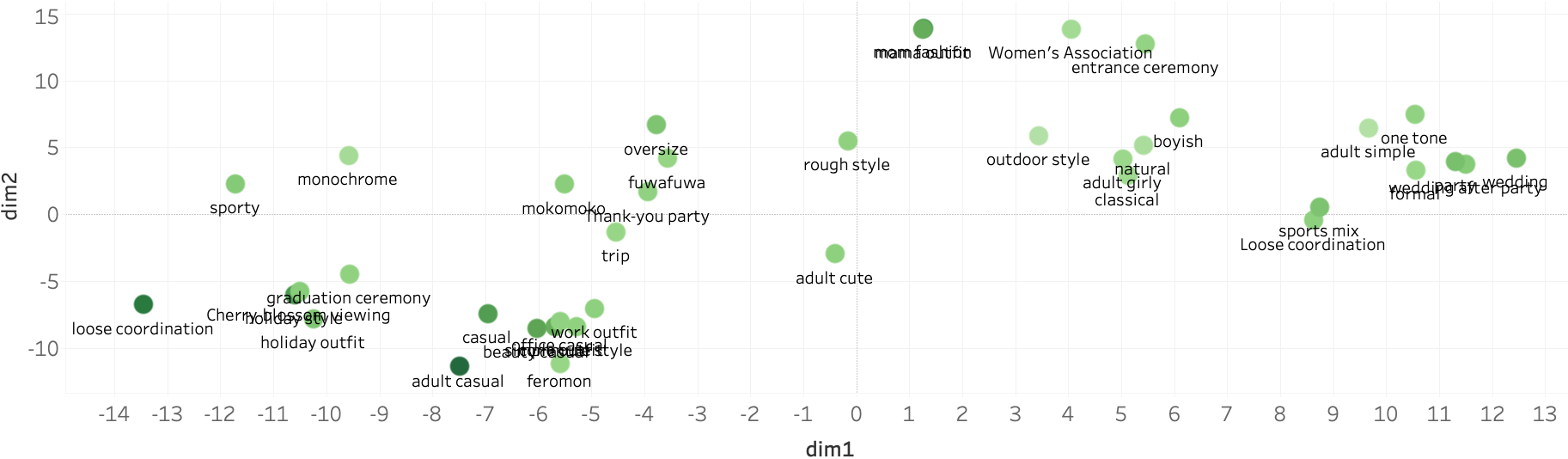}
  }
  \caption{Mapping the result of compressing the average of tag embedded representations \label{fig_result_mapping}}
 \end{center}
\end{figure}

\quad By observing these figures, it is possible to grasp the semantic relationship of each expression, taking the image information into consideration. Because the results of estimation by the three distance measures are shown, it is possible to consider the validity and interrelationships of each distance. First, in the results obtained from KL divergence and $2$-Wassestein distance models, for example, the pairs of ``wedding'' and ``wedding after-party,'' ``mom outfit'' and ``mom fashion,'' ``commute style'' and ``work outfit,'' ``holiday outfit'' and ``holiday style,'' and ``wedding'' and ``wedding after party'' that have similar meaning to each other are in close proximity. Conversely, for the Mahalanobis distance model, these pairs are in many cases not near each other, so the validity of the results is questionable. These results suggest that the adoption of the Mahalanobis distance may be risky for the model and the problem under study.

\quad For a more detailed interpretation, Figure~\ref{fig_result_mapping_zoom} below shows an enlarged map (KL divergence) of the area around many of the tags associated with ``casual.''

\begin{figure}[ht]
\centering
\includegraphics[width=0.95\linewidth]{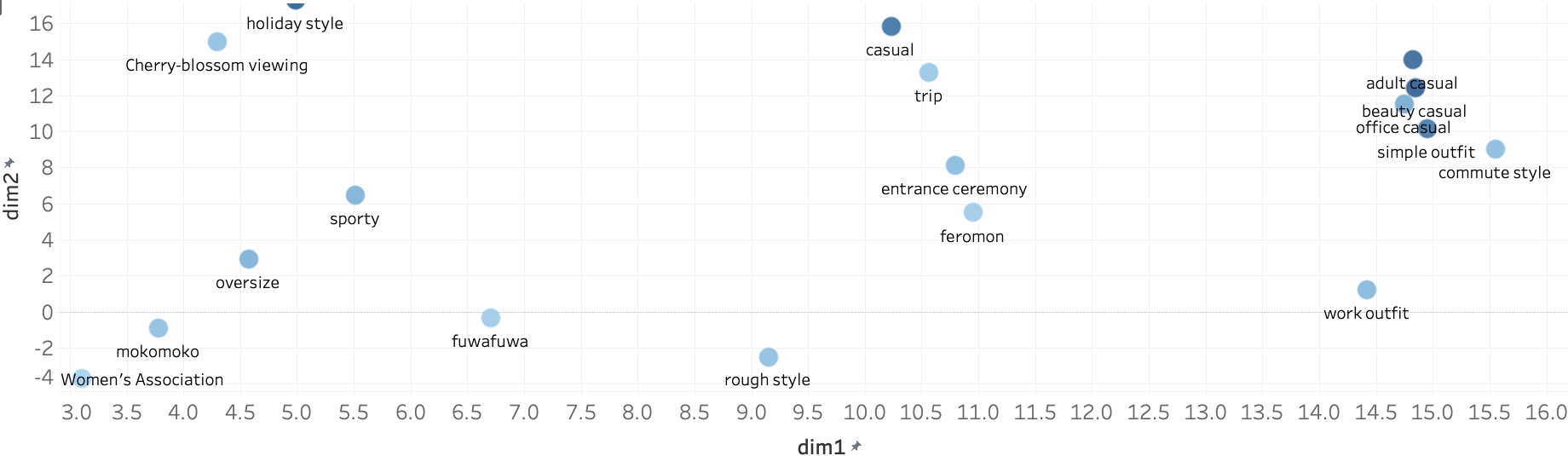}
\caption{Enlarged map of the area around many of the tags associated with ``casual'' (KL-divergence) \label{fig_result_mapping_zoom}}
\end{figure}

\quad Using this figure, it is possible to accurately understand the relationship between these abstract expressions, which have been used subjectively in the past when conversing about fashion. For example, the fact that ``office casual'' is more similar to ``beauty casual'' than to ``adult casual'' may have been an ambiguous fact for experts as well as non-specialists. It is also a new finding by quantitatively expressing each tag that ``office casual'' is closer to ``simple outfit,'' ``commute style,'' and ``work outfit'' than to ``adult casual.'' Furthermore, ``office casual,'' ``adult casual,'' and ``beautiful casual'' are quite similar in expression, and we can see that making them more ``casual'' brings them closer to ``trip'' and ``enrollment ceremony'' (suitable attire). Clearly, ``fuwafuwa'' and ``mokomoko'' are also similar and are onomatopoeic words used to describe the near meanings to like ``oversize'' clothing and ``rough style.''

\quad The use of this map-based interpretation support method will reduce the difficulties in understanding fashion-specific ambiguous expressions. It is expected that all users will be able to talk and explain fashion using words specific to the fashion field and make decisions based on a common understanding. Furthermore, it is important to note that this map is not based only on word (tag) co-occurrence relations, as in Word2Gauss and GVSE, but is obtained by considering image information.

\subsection{Interpretation of Variance}
The proposed model can acquire both the mean and the variance for embedded representations of images and tags. Although the conventional GVSE model can also acquire only variance for words, the embedded representation for the acquired words is not observed and considered. In contrast, in this study, the variance obtained is also thoroughly observed and considered.

\subsubsection{Variance for Attributes}
First, Table~\ref{table_tag_variance} below shows the top eight and bottom eight tags with the largest variance, their variance values, and the number of images (count) to which they are attached in the whole images.

\begin{table}[ht]
\centering
\caption{Summary of tags with large and small variance of embedded representation}
\label{table_tag_variance}
\scalebox{1.0}{
\begin{tabular}{| ccrr | ccrr |}
\multicolumn{8}{c}{} \\
\multicolumn{8}{c}{(a) Mahalanobis distance} \\
\hline
rank & tag & \multicolumn{1}{c}{variance} & \multicolumn{1}{c|}{count} & rank & tag & \multicolumn{1}{c}{variance} & \multicolumn{1}{c|}{count} \\
\hline \hline
1  & Ease-up Coordinates  & 1.0369 & 147 & 1095 & Petites            & 0.9144 & 2618 \\
2  & Hair Bands           & 1.0331 & 235 & 1096 & Mom's Coordinates  & 0.9137 & 811  \\
3  & Fall Colors          & 1.0280 & 86  & 1097 & Wedding            & 0.9088 & 697  \\
4  & Gaucho Pants         & 1.0258 & 74  & 1098 & Beautiful Casual   & 0.9085 & 1717 \\
5  & Big Silhouette       & 1.0254 & 194 & 1099 & Spring Coordinates & 0.9074 & 3111 \\
6  & Bun Hairstyles       & 1.0222 & 37  & 1100 & Denim              & 0.9069 & 3116 \\
7  & Floral Blouse        & 1.0221 & 14  & 1101 & Knitwear           & 0.9062 & 2792 \\
8  & Drawstring Bags      & 1.0215 & 84  & 1102 & Black              & 0.9036 & 2344 \\
9  & Ruffled Knitwear     & 1.0191 & 32  & 1103 & Mom Fashion        & 0.8926 & 952  \\
10 & Waist Mark           & 1.0183 & 25  & 1104 & Short Stature      & 0.8398 & 6714 \\
\hline

\multicolumn{8}{c}{} \\
\multicolumn{8}{c}{(b) KL-divergence} \\
\hline
rank & tag & \multicolumn{1}{c}{variance} & \multicolumn{1}{c|}{count} & rank & tag & \multicolumn{1}{c}{variance} & \multicolumn{1}{c|}{count} \\
\hline \hline
1  & Petites                  & 1.4426 & 2618 & 1095                 & Normcore          & 0.8587 & 8  \\
2  & Knitwear                 & 1.3853 & 2792 & 1096                 & Black Lace        & 0.8587 & 23 \\
3  & Adult Women              & 1.3736 & 1624 & 1097                 & tweedmill         & 0.8578 & 5  \\
4  & Fall Coordinates         & 1.3735 & 2291 & 1098                 & Beige Pants       & 0.8569 & 81 \\
5  & Otona Casual             & 1.3686 & 2759 & 1099                 & Dandy Glasses     & 0.8569 & 95 \\
6  & Little Recommend         & 1.3565 & 673  & 1100                 & Character T-shirt & 0.8554 & 26 \\
7  & Black                    & 1.3562 & 2344 & 1101                 & Red Converse      & 0.8549 & 39 \\
8  & Ballet Shoes             & 1.3495 & 1829 & 1102                 & Red Socks         & 0.8533 & 41 \\
9  & Wedding Party Style      & 1.3427 & 681  & 1103                 & O'Neill of Dublin & 0.8425 & 5  \\
10 & Winter Coordinates       & 1.3408 & 1534 & 1104                 & Gray Tights       & 0.8067 & 4 \\
\hline

\multicolumn{8}{c}{} \\
\multicolumn{8}{c}{(c) $2$-Wassestein distance} \\
\hline
rank & tag & \multicolumn{1}{c}{variance} & \multicolumn{1}{c|}{count} & rank & tag & \multicolumn{1}{c}{variance} & \multicolumn{1}{c|}{count} \\
\hline \hline
1  & Petites            & 1.4597 & 2618 & 1095                 & Danton                & 0.8490 & 20 \\
2  & Black              & 1.3552 & 2344 & 1096                 & Black Skirt           & 0.8483 & 85 \\
3  & Knitwear           & 1.3466 & 2792 & 1097                 & Cable-Knit            & 0.8428 & 6  \\
4  & Adult Women        & 1.3442 & 1624 & 1098                 & Gaucho Pants          & 0.8426 & 74 \\
5  & Fall Coordinates   & 1.3400 & 2291 & 1099                 & Beige Pants           & 0.8396 & 81 \\
6  & Adult Casual       & 1.3375 & 5680 & 1100                 & Voluminous Skirts     & 0.8393 & 22 \\
7  & Ballet Shoes       & 1.3157 & 1829 & 1101                 & ZOZO Summer Sale      & 0.8370 & 36 \\
8  & Otona Casual       & 1.3140 & 2759 & 1102                 & Normcore              & 0.8347 & 8  \\
9  & Loose Coordinates  & 1.3069 & 1558 & 1103                 & High Waist Pants      & 0.8233 & 33 \\
10 & Denim              & 1.2974 & 3116 & 1104                 & Fake Glasses          & 0.7748 & 95 \\
\hline
\end{tabular}
}
\end{table}

\quad First, a comparison of the tables obtained by the three measures shows that many of the tags judged to have particularly high variance are overlapped, in DGVSE models based on KL divergence and $2$-Wassestein distance. Conversely, the results for the Mahalanobis distance model differ significantly from those of the other two models. Combined with the results for the mean mentioned in the previous section, this suggests that KL divergence and $2$-Wassestein distance are similar measures and that Mahalanobis distance may be a very different measure compared to the other two measures.

\quad The results of KL divergence and $2$-Wassestein distance show that the variance tends to be larger for tags with a larger frequency of occurrence (count) and smaller for tags with a smaller frequency of occurrence (count). The top tags mostly contain abstract attributes, colors, and highly versatile items, such as ``denim,'' while the bottom tags mostly contain specific items that are not included in many coordinates. It is natural that the variance increases for highly versatile items and colors (specific tags,) because they are included in a variety of coordinates (i.e., assigned to many images). Therefore, a certain degree of correlation with the number of times a tag is assigned may be a basis for increasing the validity of the results. For example, the tag ``petites'' is not assigned to fashion items or atmosphere, but to images of models with a small height. Therefore, the variance of tags, such as ``knitwear,'' ``denim,'' ``adult casual,'' and ``otona casual,'' that are directly related to fashion are smaller than that of tags, despite these tags appearing more frequently than ``petites.'' This suggests the validity of the obtained variance.

\subsubsection{Variance for Images}
We discuss the variance of the images obtained from DGVSE. It was difficult to understand the relationship between images and dispersion only by observing the images. Therefore, we compared the variance estimated by DGVSE with each statistic related to the tags assigned to the images. The results are shown in the following Figure~\ref{fig_result_corr_cov}.

\begin{figure}[ht]
\centering
\subfigure[Mahalanobis distance]{
   \includegraphics[width=.90\columnwidth]{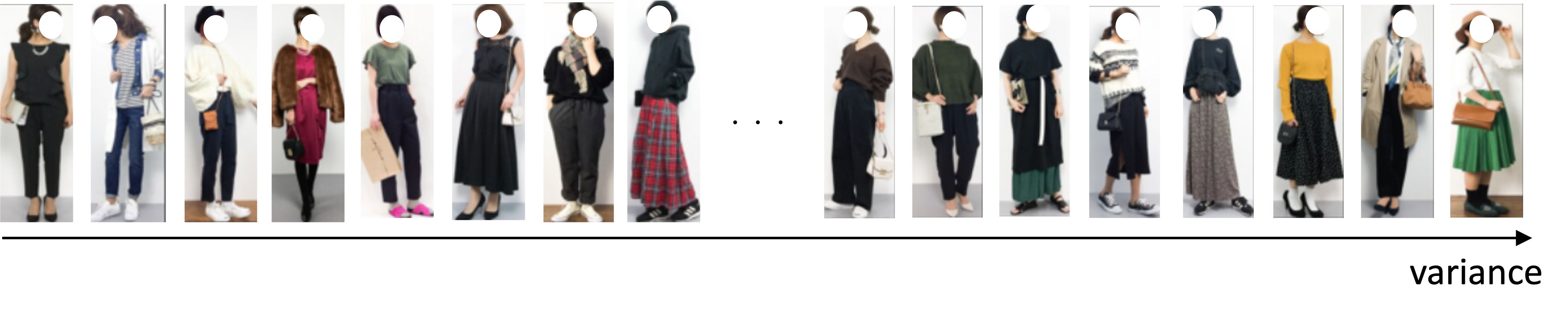}
   \label{fig_result_cov_ma}
} \\
\subfigure[KL-divercenge]{
   \includegraphics[width=.90\columnwidth]{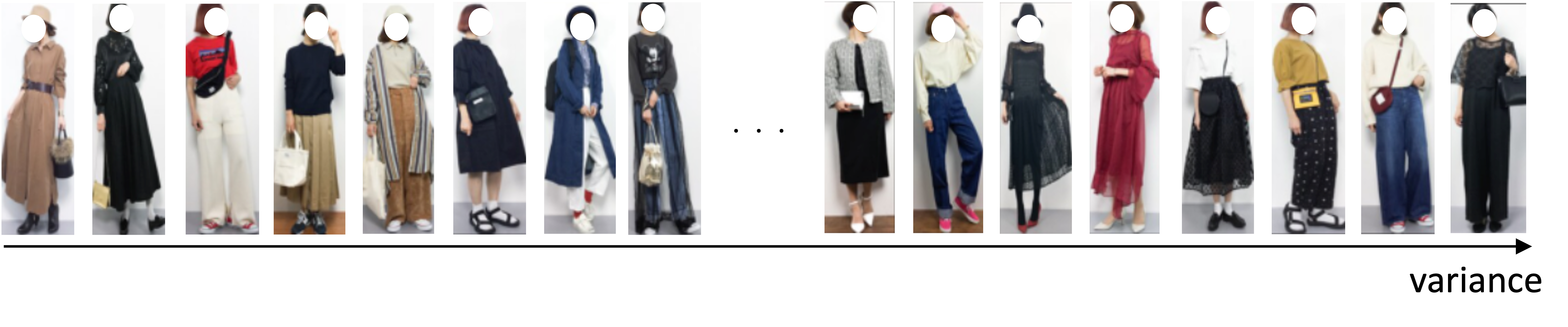}
   \label{fig_result_cov_kl}
} \\
\subfigure[$2$-Wassestein distance]{
   \includegraphics[width=.90\columnwidth]{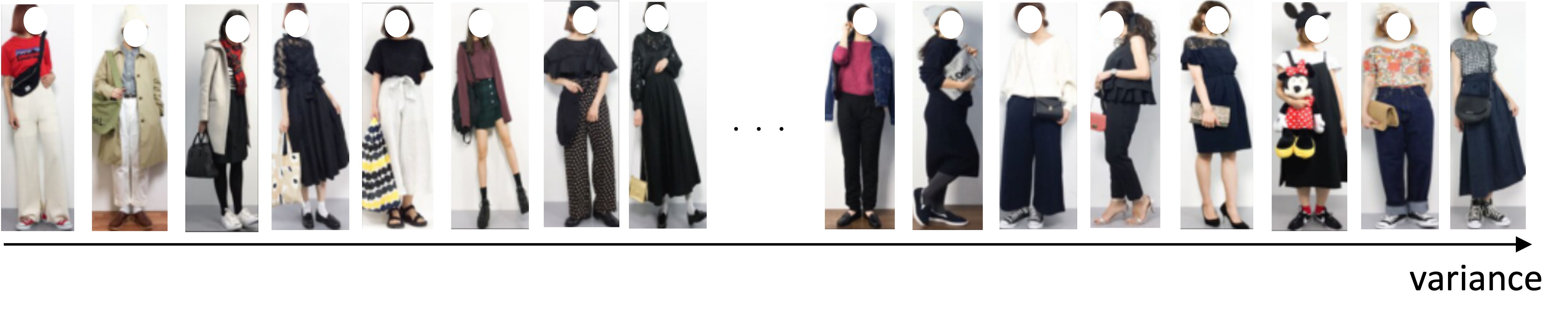}
   \label{fig_result_cov_ws}
} \\
\caption{Images sorted by variance of image embedded representation}
\label{fig_result_cov}
\end{figure}

\quad By observing this figure, it is possible to see which images are more semantically broad and which are narrower. However, it was difficult to understand the relationship between images and dispersion just by observing the images. Therefore, we compared the variance estimated by DGVSE with each statistic related to the tags assigned to the images. The results are shown in the following Figur~\ref{fig_result_corr_cov}.

\begin{figure}[ht]
 \begin{center}
  \subfigure[Mahalanobis distance]{
   \includegraphics[width=.48\columnwidth]{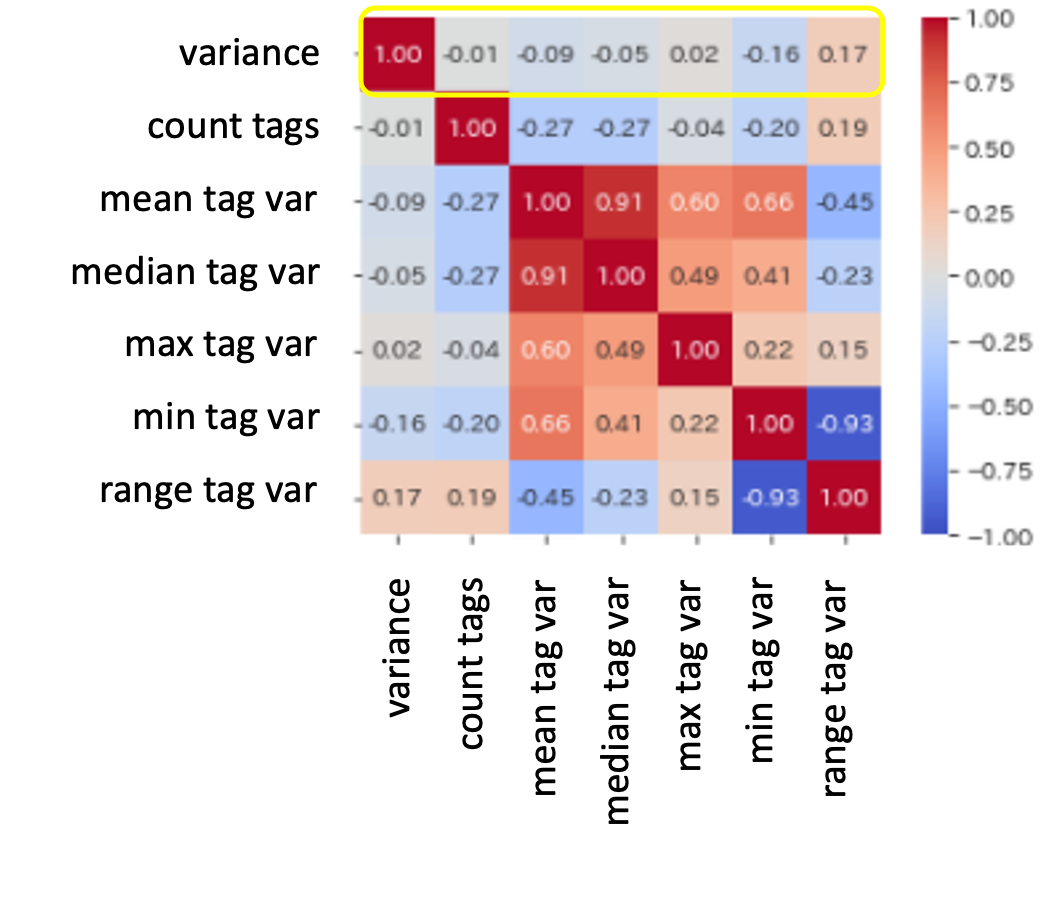}
  } \\
  \subfigure[KL-divergence]{
   \includegraphics[width=.48\columnwidth]{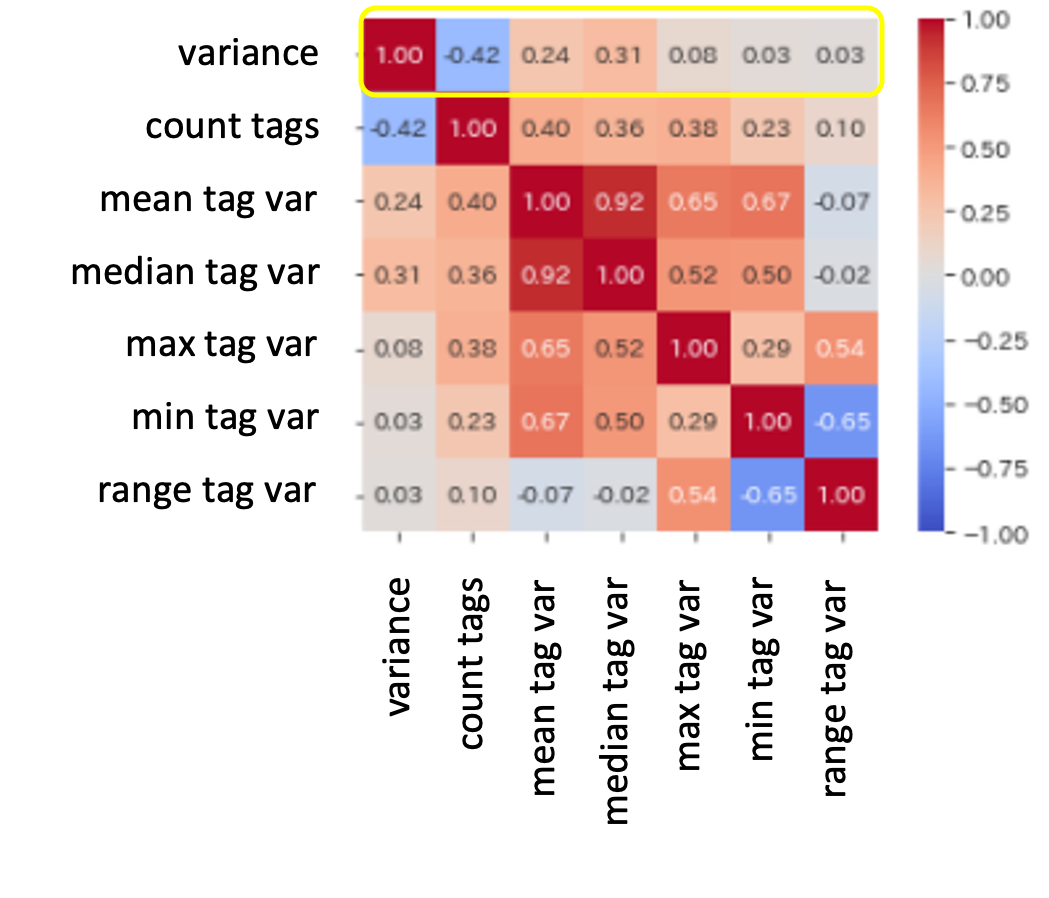}
  }~
  \subfigure[$2$-Wassestein distance]{
   \includegraphics[width=.48\columnwidth]{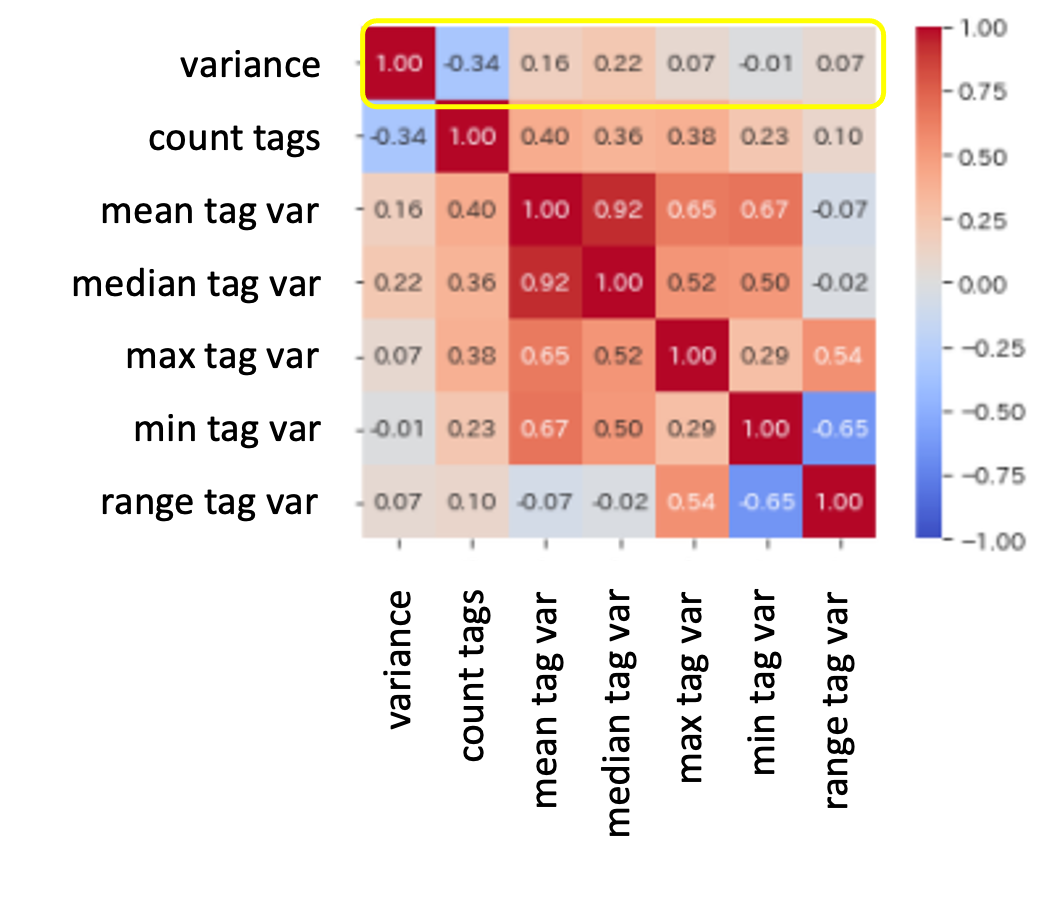}
  }
  \caption{Correlation coefficient matrix with each statistic associated with the tags assigned to the image \label{fig_result_corr_cov}}
 \end{center}
\end{figure}

\quad Results show that the variance of images obtained from the model adopting KL divergence and $2$-Wassestein distance is negatively correlated with the number of tags attached to the image, and no significant correlation was observed with other statistics. Conversely, the model adopting Mahalanobis distance showed little correlation with any of the statistics. This result can be attributed to the fact that when training the model, the placement of images with many tags in the embedding space is defined in detail by the tags. Conversely, images with few tags are ambiguous in terms of the placement specified by the tags, suggesting that they may require to be positioned well in relation to other images and tags. Some tags that should potentially be attached are often not assigned because tags are assigned by the general user. It is expected that the results on image variance obtained from DGVSE will be useful in understanding whether the tags assigned are insufficient to describe the target fashion image.

\quad These results suggest that in these situations, where the types and number of tags assigned to each image are diverse, the variance estimated for an image tends to be determined by the amount of information that can be obtained from the tags information for the model (ambiguity of the tags as a whole). In this study, we used the accumulated data as is, but it may be possible to learn a better embedding space by implementing tag completion and other measures. In such a case, the proposed method is expected to be used to focus on images with large variance and tag completion to reduce the variance and ultimately contribute to acquiring a delicate and accurate embedding space.

\subsection{Image Retrieval}
DGVSE allows for the estimation of the variance of the embedded representation of each image or tag and also maintains the original application methods of VSE in~\cite{Shimizu2022_FashionIntelligenceSystem}.

\quad Figure~\ref{fig_retult_retrieval} below shows an example of image retrieval results based on tag and image operations. We present results using a model adopting the $2$-Wassestein distance (out of the KL divergence model and $2$-Wassestein distance model), which is able to learn a reasonable embedding space based on the multifaceted analysis described above.

\begin{figure}[ht]
\centering
\includegraphics[width=0.975\linewidth]{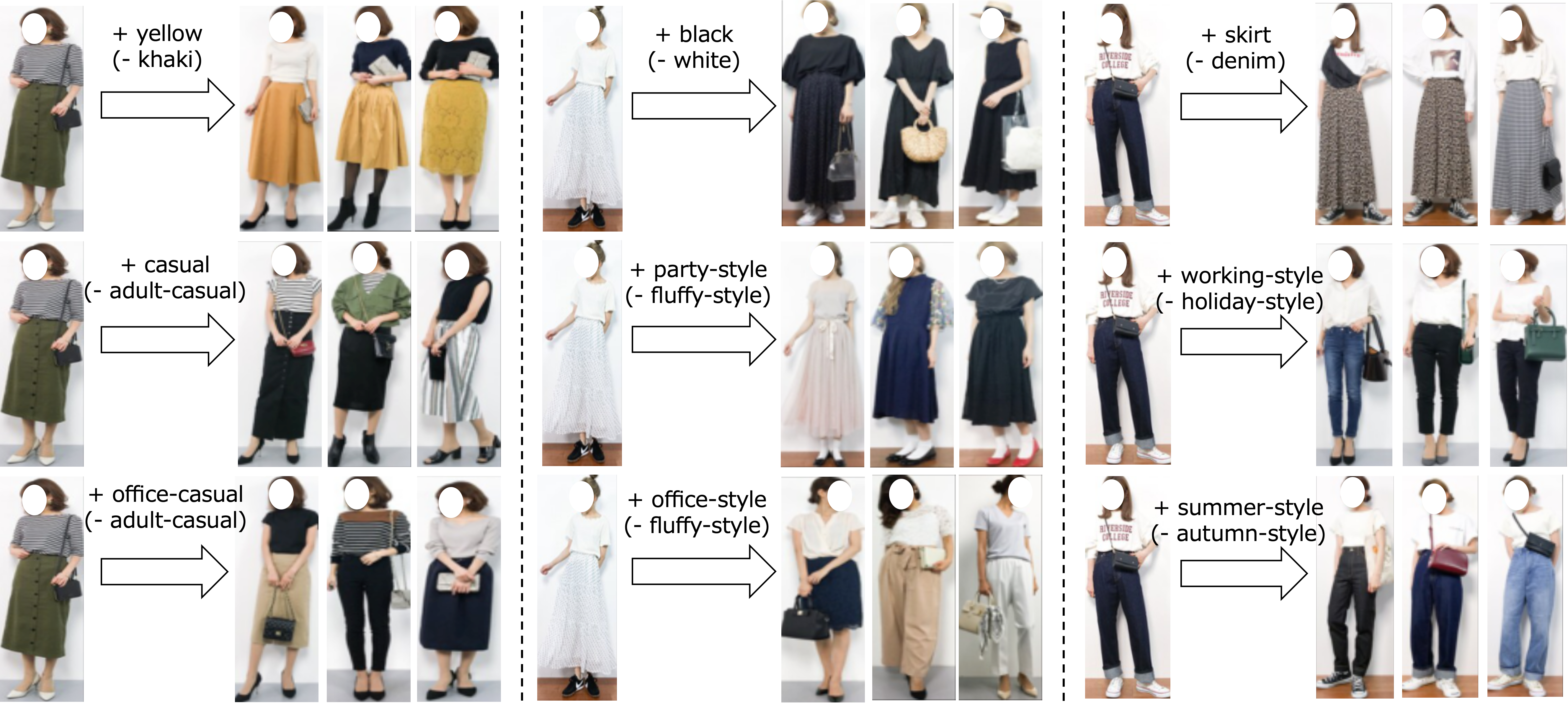}
\caption{Example of image retrieval ($2$-Wassestein distance) \label{fig_retult_retrieval}}
\end{figure}

\quad First, for clarity, the above three examples of using the image retrieval function with specific tags are shown. For example, if we remove the ``khaki'' tag from a query image (wearing a striped shirt and a khaki-colored skirt) that has been assigned the ``khaki'' tag and the ``adult casual'' tag and instead assign the ``yellow'' tag to the image with a yellow skirt, the image with a yellow skirt is retrieved, and this full-body outfit image stays in the adult casual category. In addition, the example in the two lines below shows the search results using abstract tags. For example, if the ``working-style'' tag is added instead of ``holiday-style'' to the image (on the right) with the tags ``denim,'' ``holiday-style,'' and ``autumn-style,'' it can search for images that are appropriate for the office while retaining the overall color scheme. Similarly, if the ``summer-style’’ tag is added instead of ``autumn-style,’’ the whole color tone is kept, and outfits with short sleeves or thin-looking fabrics are retrieved.

\quad Thus, DGVSE enables image retrieval utilizing tag and image operations, including fashion-specific abstract tags. 

\subsection{Image Re-ordering}
\quad Another function maintained from the original VSE is image re-ordering, which retrieves images that are more (less) relevant to the target tag by reordering images in the order of the strength of the relevance score calculated between the target tag and each image to which the tag is attached. An example of image re-ordering results using DGVSE is shown in Figure~\ref{fig_retult_reorder} below.

\begin{figure}[ht]
\centering
\includegraphics[width=0.975\linewidth]{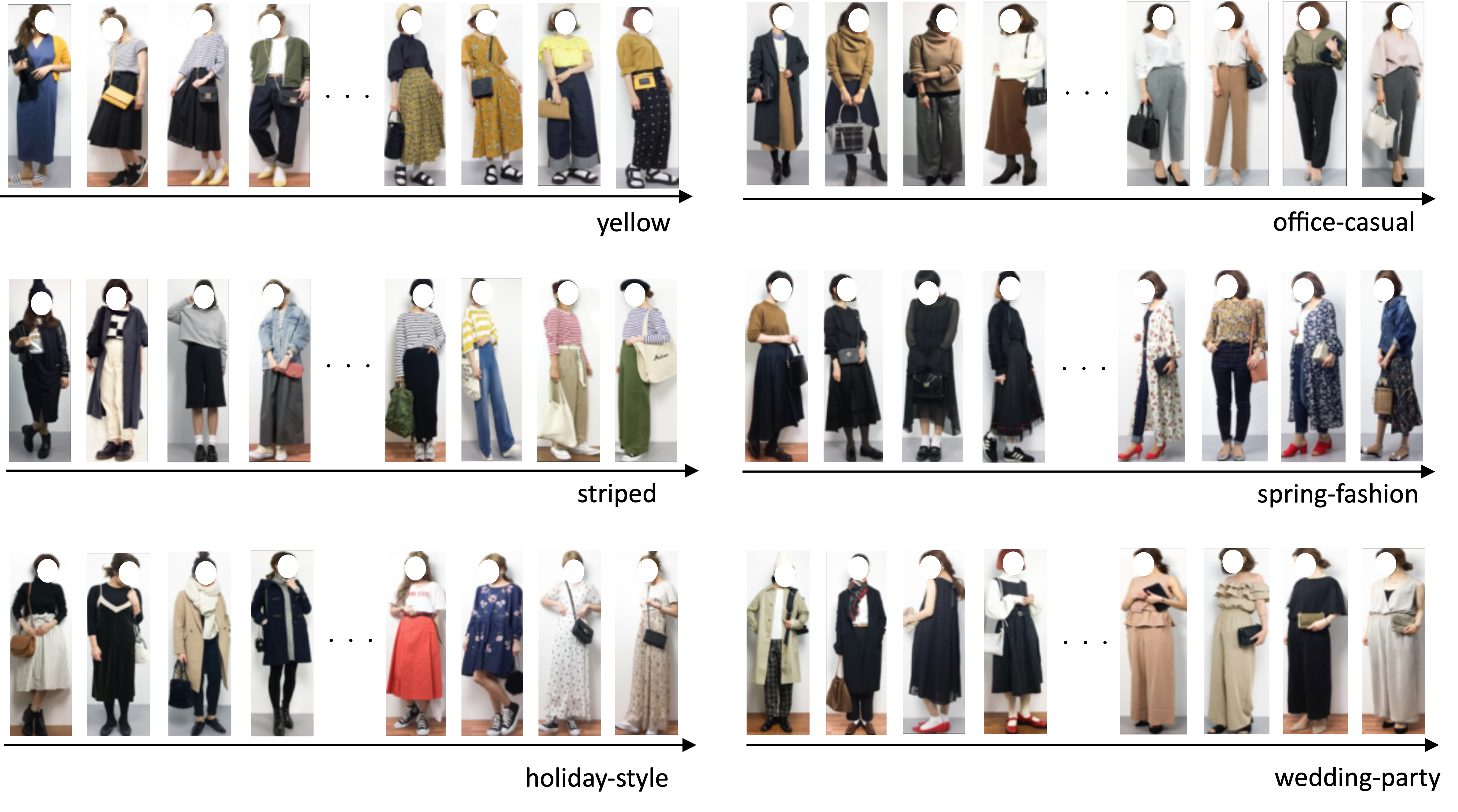}
\caption{Example of image re-ordering ($2$-Wassestein distance) \label{fig_retult_reorder}}
\end{figure}

\quad For example, outfits with a high relevance score to the specific ``yellow'' tag tend to have a high percentage of yellow areas in the entire image. Conversely, an outfit with a low relevance score will have fewer yellow areas. Other outfits with high relevance scores to the specific ``striped'' tag tend to have a high percentage of the border portion of the item in the whole image. Previously, it was only possible to display images with the ``yellow'' or ``striped'' tag in a batch; however, this application allows users to search for clothes that meet their detailed objectives such as ``I want to incorporate yellow in only one item'' or ``I want to find clothes that contain a border pattern throughout the entire outfit.” It is now possible to search for clothes in detail according to the  detailed objectives of the user.

\quad The results of image reordering using abstract expressions, such as ``holiday style,'' ``office casual,'' ``spring fashion,'' and ``wedding,'' and the relevance scores of tags related to usage scenes and seasons also have the potential to be used very effectively in actual services. For example, the results of ``holiday style, images with strong relevance to ``holiday style,'' are generally warm in color and often contain items with patterns, soft silhouette clothes and sneakers. Conversely, images with low relevance scores include items in specific colors, such as black and beige, and slightly more formal (office casual) items such as tights and heels. The images with high relevance to ``spring style'' contained brightly colored and patterned items, while the images with low relevance contained a lot of black. Using these results, users can ask themselves questions such as ``what is a good outfit for the holidays?'' and ``what is appropriate for spring?'' as well as ``what is the best outfit for a holiday?''

\quad Thus, DGVSE can reorder images utilizing tag and image operations, including fashion-specific abstract tags. These functions are expected to be effective in quantifying fashion-specific abstract terms and images and to support the interpretation of fashion by users.

\quad While the functions of image retrieval and image sorting are maintained from the conventional VSE, clearly, the proposed method has a wider range of applications than the conventional methods, because it also allows interpretation using variance as mentioned above. This is the strength of DGVSE in this application, and it is expected to be more widely effective in supporting user understanding of fashion in real applications.

\section{Discussion}
\subsection{On the Loss Function}
In section~\ref{sec:theory_mahalanobis}, we theoretically showed the risk that the Mahalanobis distance between points (images) and probability distributions (attributes) adopted in GVSE cannot implement stable learning. Moreover, by assuming variance even for images and by making it possible to measure the distance between variance against variance, many measures can be introduced. In this study, we adopted the Mahalanobis distance, KL divergence, and $2$-Wassestein distance and confirmed the validity and appropriateness of each measure through multifaceted experiments. Based on these results, it is necessary to further investigate which distance should be adopted.

\subsubsection{On the Risks of the Mahalanobis Distance}
The results of the mapping experiment using the mean values of the embedded representations obtained showed that 1) KL divergence and the $2$-Wassestein distance model produced reasonable maps in which similar expressions (e.g., ``mom-cordinate'' and ``mom-fashion'') were gathered in close proximity, while the Mahalanobis distance model produced a questionable mapping. 2) Similarly, the results for the variance of the tags obtained show that KL divergence and the $2$-Wassestein distance model have a larger variance for abstract expressions and for items that are included in many coordinates and slight for items that are not included in many coordinates, indicating that the results were somewhat valid. Conversely, the Mahalanobis distance model gave quite different results. 3) The results for the variance of the acquired images from KL divergence and the $2$-Wassestein distance model differed significantly from those of the Mahalanobis distance model.

\quad The Mahalanobis distance between the two probability distributions is expressed by Eq.~\eqref{eq_d_mahalanobis}. If we look at the part of the joint covariance matrix $\Sigma$ in this equation, we can see that it plays only a scaling role when taking the inner product of the differences of mean vectors between the two distributions. Specifically, when this part (the sum of the two distributional covariance matrices) is large, the overall Mahalanobis distance increases, and when it is small, the overall Mahalanobis distance decreases. Moreover, it plays no other role. This makes it questionable whether it is appropriate to adopt this type of distance to understand the spread of meanings of images and tags. Particularly, in the case of data similar to that used in this study, where the frequency of occurrence of each tag is highly skewed, the loss of the entire training batch containing the tag can be reduced by reducing the variance of the tag with the highest frequency of occurrence. Specifically, it is expected that learning will proceed in such a way that the variance is adjusted by looking at the frequency of occurrence of tags attached to images in the entire data set. Moreover, the experimental results showed a complete tendency to do this. This suggests that the adoption of the Mahalanobis distance may be risky for this target problem and for the proposed DGVSE model.

\subsubsection{On the Similarities Between KL divergence and 2-Wassestein Distance}
The respective results obtained from the models adopting KL divergence and $2$-Wassestein distance were very similar.

\quad It is known that the KL divergence measures the distance from the midpoints for multimodal distributions.
More precisely, given $p\in\mathcal{P}$, we search for the distribution $\hat{p}$ that minimizes the divergence from $p$ to a smooth submanifold $\mathcal{S}\subset\mathcal{P}$,
\begin{align}
    \hat{p} = \underset{q\in\mathcal{S}}{\arg\min} D_{KL}[p\|q].
\end{align}
Then, the best approximation $\hat{p}$ in the closure of $\mathcal{S}$ satisfies $\hat{p}(\bm{x}) = 0$ for $\bm{x}$ at which $p(\bm{x})=0$, and this property is called zero-forcing.
Note that for the reverse KL divergence $D_{KL}[q\|p]$, the best approximation $\hat{p}$ in the closure of $\mathcal{S}$ satisfies $\hat{p}(\bm{x})\neq 0$ for $x$ at which $p(\bm{x})\neq 0$.
On the other hand, $l$-Wassestein distance is robust for multimodal distributions~\cite{arjovsky2017wasserstein, kolouri2018sliced}. 

However, there was no significant difference between the results obtained from both models in this study. Therefore, it is possible that the distributions estimated as embedded representations of each tag or image are unimodal, at least for the data of this study. It is also suggested that the distance measures selected in this study (KL divergence and $2$-Wassestein distance) do not have any major problems. Thus, the distance to be used should be appropriately selected according to the target problem and the characteristics of the data, while observing the theoretical and empirical results.

\subsection{On the Covariance Matrix}
In this study, we assumed a multidimensional Gaussian distribution with only the diagonal components having values behind the embedded representation. That is, the parameters were estimated assuming that each component follows an independent Gaussian distribution. This method is used in many Gaussian embedding methods such as Word2Gauss and GVSE. If the non-diagonal components also have values, the number of parameters increases to ``(\{number of images\} $+$ \{number of tags\}) $\times$ \{number of dimensions of the embedded representation\}$^2$.'' On the other hand, if only the diagonal components have values, the number of parameters can be reduced to ``(\{number of images\} $+$ \{number of tags\}) $\times$ \{number of dimensions of the embedded representation\}.''

\quad This study assumes a ``spherical'' situation where all diagonal components have the same value. This method reduces the number of parameters to ``\{number of images\} $+$ \{number of tags\}.'' This method is used in many Gaussian embedding methods such as Word2Gauss and GVSE. On the other hand, a ``diagonal'' situation in which all diagonal components have different values can also be considered, but the number of parameters increases to ``(\{number of images\} $+$ \{number of tags\}) $\times$ \{number of dimensions of the embedded representation\}.'' Even if the analyst adopts the diagonal covariance matrix and different variances are obtained for each dimension, the marketing insight that can be gained from the results is small. In other words, the analyst wants to know one value per image or tag (the overall trend) and is unlikely to gain much benefit if different values are obtained for each dimension individually.

\quad Therefore, the method of estimating a ``spherical'' covariance matrix adopted in this study is considered the best in achieving the objective of this study, which is to propose a method that is useful for marketing.

\subsection{On the Strengths of the Proposed Model}
Because VSE, GVSE, and DGVSE have been proposed in previous studies, it is necessary to clarify the differences between them and to understand which is stronger in which situations. The following table summarizes the characteristics of each method.

\begin{table}[ht]
\centering
\caption{Summary comparing the characteristics of each method}
\label{table_feature_vses}
\scalebox{0.90}{
\begin{tabular}{c||c|c|c}
\hline
                                  & VSE & GVSE  & DGVSE \\ \hline \hline
embedding of image (mean)         & $\bigcirc$ & $\bigtriangleup$ (without visual information) & $\bigcirc$ \\
embedding of attribute (mean)     & $\bigcirc$ & $\bigcirc$ & $\bigcirc$     \\
embedding of image (variance)     & $\times$   & $\bigcirc$ & $\bigcirc$     \\
embedding of attribute (variance) & $\times$   & $\times$   & $\bigcirc$     \\
\hline
\end{tabular}
}
\end{table}

\quad As shown in the above table, DGVSE is the one that can capture more from a single model. On the other hand, VSE has the advantage of requiring fewer parameters to be estimated, although it is not able to capture variance. If the purpose is only to use the original functions of VSE, such as search and sorting, the choice should be based on the accuracy in the target task. However, because the purpose of this study is to create term maps and visualize the spread of meanings, DGVSE is suitable for this purpose.

\section{Conclusions}
In this study, we proposed a new model, DGVSE, which is a kind of new fashion intelligence system that enables interpretation of abstract fashion attribute information. The proposed model can be added to the VSE function of mapping images and tags in the same space and can estimate the variance of both embedded representations. The conventional method of GVSE faces the following problems: 1) it is not an end-to-end learning method, 2) it only observes the co-occurrence of words and ignores image information when learning the embedded representation (mean) of words, and 3) it assumes a multi-dimensional Gaussian distribution only for words (hence the problem of measuring the distance between a point and the probability distribution using the Mahalanobis distance); the proposed model solves these problems. The risk that loss functions that include the Mahalanobis distance cannot be learned stably has been theoretically and empirically demonstrated. We showed how multifaceted analysis contributes to reducing fashion-specific ambiguity and complexity by applying the proposed method to a real service dataset. We expect that this method can be used in the real world to develop systems that support user purchasing activities in online fashion search.

\subsection{Limitations}
A quantitative assessment of the performance of the models would have been ideal. However, conducting a quantitative evaluation requires a large annotation effort on a subjective and abstract issue that involves the mobilization of experts; therefore, we avoided that task. Instead, theoretical validation and multifaceted qualitative evaluation experiments demonstrated the effectiveness of the proposed model. Theoretical verification of the similarity and difference between $2$-Wassestein and KL divergence is also included in the limitations.

\subsection{Future works}
Future research direction include conducting quantitative evaluation experiments for abstract and difficult-to-evaluate problems, such as those covered in this study, creating datasets that enable these experiments (e.g., inspired by~\cite{kimura2021shift15m}), examining models that assume probability distributions other than the multidimensional Gaussian distribution, and further analysis on the best distance measure to use~\cite{renyi1961measures,amari2009alpha,kimura2021alpha}.
In particular, the framework of information geometry~\cite{amari2000methods,amari2000methods,ay2017information}, which considers Riemannian manifolds formed by probability distributions, is very useful, and many machine learning algorithms have been analyzed~\cite{amari1995information,amari1998natural,murata2004information,gaussian_riemannian_face_regression,kimura2022information,kimura2021generalized}. We aim to better understand the behavior of the algorithm by using these mathematical tools.

\section*{Competing Interests}
The authors declare no conflicts of interest.

\section*{Acknowledgements}
This work was supported by JSPS KAKENHI Grant Number 21H04600.


\end{document}